\documentclass[journal]{IEEEtran}
\usepackage[pdftex]{graphicx}
\setkeys{Gin}{clip=true,draft=false}
\DeclareGraphicsExtensions{.pdf}
\usepackage[cmex10]{amsmath}
\usepackage{amsthm}
\usepackage{amsfonts}
\usepackage[cspex,bbgreekl]{mathbbol}
\usepackage[hang]{subfigure}
\usepackage{color}
\usepackage{braket}
\usepackage{yfonts}
\usepackage{amssymb}
\usepackage[inline]{enumitem}
\newcommand{\Mod}[1]{\ (\mathrm{mod}\ #1)}

\newtheorem{definition}{Definition}
\newtheorem{theorem}{Theorem}[section]
\newtheorem{corollary}{Corollary}[theorem]
\newtheorem{lemma}[theorem]{Lemma}

\begin{document}

\title{On the Universality of Memcomputing Machines}


\author{Yan Ru Pei, Fabio L. Traversa, and Massimiliano Di Ventra\thanks{YRP and MD are with the Department of Physics, University of California-San Diego, 9500  Gilman Drive, La Jolla, California 92093-0319, USA, (e-mail: yrpei@ucsd.edu, diventra@physics.ucsd.edu). FLT is with MemComputing Inc., San Diego, CA 92130 USA (e-mail: ftraversa@memcpu.com). }}

\maketitle

\begin{abstract}
Universal memcomputing machines (UMMs) [IEEE Trans. Neural Netw. Learn. Syst. {\bf 26}, 2702 (2015)] represent a novel computational model in which memory (time non-locality) accomplishes both tasks of storing and processing of information. UMMs have been shown to be Turing-complete, namely they can simulate any Turing machine. In this paper, we first introduce a novel set-theory approach to compare different computational models, and use it to recover the previous results on Turing completeness of UMMs. We then relate UMMs directly to liquid-state machines 
(or ``reservoir computing'') and quantum machines (``quantum computing''). 
We show that UMMs can simulate both types of machines, hence they are both ``liquid-'' or ``reservoir-complete'' and ``quantum-complete''. Of course, these 
statements pertain only to the type of problems these machines can solve, and not to the amount of resources required for such simulations. Nonetheless, the 
set-theoretic method presented here provides a general framework in which to describe the relation between any computational models. 

\end{abstract}


\section{Introduction}

{\it Memcomputing} stands for ``computing {\it in} and {\it with} memory'' \cite{13_memcomputing}. It is a novel computing paradigm whereby memory (time non-locality) is employed to perform both the storage and the processing of information on the same physical location. This paradigm is substantially different than the one implemented in our modern-day computers \cite{complexity_2}. In our present computers, there is a separation of tasks between a memory (storage) unit and an information processor. Modern-day computers are the closest physical approximation possible to the well-known (mathematical) Turing paradigm of computation that maps a finite string of symbols into a finite string of symbols in discrete time \cite{complexity}.

The memcomputing paradigm has been formalized by Traversa and Di Ventra in \cite{umm}. In that paper, it was shown that universal memcomputing machines (UMMs) can be defined as digital (so-called {\it digital memcomputing machines} (DMMs) \cite{dmm2}) or analog \cite{traversaNP}, with the digital ones offering an easy path to scalable machines (in terms of size) \cite{leverage}. 

UMMs have several features that make them a powerful 
computing model, most notably {\it intrinsic parallelism}, {\it information overhead}, and {\it functional polymorphism}~\cite{umm}. The first feature means that they can operate {\it collectively} on all 
(or portions of) the data at once, in a cooperative fashion. This is reminiscent of neural networks, and indeed neural networks can be viewed as special cases of UMMs. However, neural networks do not have ``information overhead''. This feature is related to the {\it topology} (or architecture) of the network of memory units ({\it memprocessors}). It means the machine has access, at any given time, to more information (precisely originating from the topology of the network) than what is available if the memprocessors were not connected to each other. Of course, this information overhead is not necessarily stored by the machine (the {\it stored} information is the Shannon one)~\cite{umm}. 
Nevertheless, with appropriate topologies, hence with appropriate information overhead, UMMs can solve complex problems very efficiently~\cite{umm}. Finally, functional polymorphism means that the machine is able to compute different functions by simply applying the appropriate input signals~\cite{umm}, without modifying the topology of the machine network.

In Ref.~\cite{umm} it was shown that UMMs are Turing-{\it complete}, meaning that UMMs can simulate {\it any} Turing machine. Note that Turing-completeness means that all problems a Turing machine can solve, can also be solved by a UMM. However, it does {\it not} imply anything about the resources (in terms of time, space and energy) required to solve those 
problems. 

The reverse, namely that all problems solvable by UMMs are also solvable by Turing machines has not been proven yet. Nevertheless, a practical realization of {\it digital} memcomputing machines, using electronic circuits with memory~\cite{dmm2}, has been shown to be efficiently simulated with our modern-day computers (see, e.g.,~\cite{max-sat}). In other words, the ordinary differential equations representing DMMs and the problems they are meant to solve can be efficiently simulated on a classical computer. 

In recent years, other computational models, each addressing different types of problems, have attracted considerable interest in the scientific community. On the one hand, {\it quantum computing}, namely computing using features, like tunneling or entanglement, that pertain only to quantum systems, has become a leading candidate to solving specific problems, such as prime factorization~\cite{shor}, annealing \cite{anneal}, or even simulations of quantum Hamiltonians~\cite{q_hamiltonian}. This paradigm 
has matured from a simple proposal~\cite{benioff_1980} to full-fledged devices for possible commercial use~\cite{dwave_1}. However, its scalability faces considerable practical challenges, and its range of applicability is very limited compared to even classical Turing machines. 

Another type of computing model pertains to a seemingly different domain, the one of spiking (recurrent) neural networks, in which spatio-temporal patterns of the network can be used to, e.g., learn features after some training~\cite{liquid_1}. Although somewhat different realizations of this type of networks have been suggested, for the purposes of this paper, here we will focus  only on the general concept of ``reservoir-computing''~\cite{liquid_3}, and in particular on its ``liquid-state machine'' (LSM) realization~\cite{liquid_4}, rather than the ``echo-state network'' one~\cite{liquid_2}. The results presented in this paper carry over to this other type of realization as well. Therefore, we will use the term ``reservoir-computing'' as analogous to ``liquid-state machine'' and will not differentiate among the different realizations of the former. 

Our goal for this paper is to study the relation between these seemingly different computational models directly, without reference to the Turing model. In particular, we will show that UMMs encompass both quantum machines and liquid-state machines, in the sense that, on a theoretical level, they can simulate {\it any} quantum computer or {\it any} liquid-state machine. Again, this does not imply anything about the {\it resources} needed to simulate such machines. In other words, we prove here that UMMs are not only Turing-complete, but also ``liquid-complete'' (or ``reservoir-complete'') and ``quantum-complete".

In order to prove these statements, we will introduce a general \textit{set-theoretical} approach to show that the LSM and quantum computers can be mapped to subsets of the UMM. This set-theoretical approach is however very general and it is applicable to any other type of computational model, other than the ones we consider in this work.

Our paper is organized as follows. In Sec.~\ref{definition} we briefly review the definitions of the machines we consider in this work, starting with UMMs. In Sec.~\ref{mathtools} we introduce the basic mathematical ingredients that are needed to prove the completeness statements that we will need later. In Sec.~\ref{setcomp} we define a general computing model within a set-theoretical framework, and construct the equivalence and completeness relations under this framework. In Sec.~\ref{SET-rewrite}, we will re-write the definitions of UMMs, quantum computers and liquid-state machines using the results from previous sections. This allows us to show explicitly in Sec.~\ref{universality} that UMMs are not only Turing-complete, but also quantum-complete and liquid-state complete. In Sec.~\ref{discussions}, we make clear what this work implies and what it does not imply. We will also provide references to works where the simulations of the DMM subclass of UMMs have already shown significant performance advantages over other computing models for hard combinatorial optimization problems. In Sec.~\ref{conclusions} we offer our conclusions and thoughts for future work. 

Note that most of the results from Sec.~\ref{setcomp} onward are original results. Most notably, we construct the notion of a \textit{transition function} ``from the ground up" without any dependence on objects such as tape, symbols, or head. This provides the basis for a general methodology for showing equivalence and completeness relations between different computing models.

\section{Review of Machine Definitions}\label{definition}

We first provide a brief review of the definitions of the three machines we will be discussing in this paper, so the reader will have a basis of reference. 

\subsection{Universal Memcomputing Machines}

The UMM is an ideal machine formed by interconnected memory cells (``memcells'' for short or memprocessors). It is set up in such a way that it has the properties we have anticipated of intrinsic parallelism, functional polymorphism, and information overhead \cite{umm}. 


We can define a UMM as the eight-tuple~\cite{umm}:
\begin{equation}\label{UMMdef}
(M, \Delta, P, S, \Sigma, p_0, s_0, F)
\end{equation}
where $M$ is the set of possible states of a single memory cell, and $\Delta$ is the set of all transition functions:
\begin{equation}
\delta_a: M^{m_a}\setminus F \times P\rightarrow M^{m'_a}\times P^2\times A
\end{equation}
where $m_a$ is the number of cells used as input (being read), $F$ is the final state, $P$ is the set of input pointers, $m'_a$ is the number of cells used as output (being written), $P^2$ is the Cartesian product of the set of output pointers and the set of input pointers for the next step, and $A$ is the set of indices $a$. 

Informally, the machine does the following: every transition function has some label $a$, and the chosen transition function directs the machine to what cells to read as inputs. Depending on what is being read, the transition function then writes the output on a new set of cells (not necessarily distinct from the original ones). The transition function also contains information on what cells to read next, and what transition function(s) to use for the next step.

Using this definition, we can better explain what the two properties of intrinsic parallelism and functional polymorphism mean. Unlike a Turing machine, the UMM does not operate on each input cell individually. Instead, it reads the input cells as a whole, and writes to the output cells as a whole. Formally, what this means is that the transition function cannot be written as a Cartesian product of transition functions acting on individual cells, namely $\delta(\prod_i M_i)\neq\prod_i(\delta_i(M_i))$. This is the property of intrinsic parallelism. We will later show that the set of transition functions without intrinsic parallelism is, in fact, a small subset of the set of transition functions with intrinsic parallelism.

Furthermore, the transition function of the UMM is dynamic, meaning that it is possible for the transition function to change after each time step. This is shown as the set $A$ at the output of the transition function, whose elements indicate what transition functions to use for the next time step. This is the property of functional polymorphism, meaning that the machine can admit multiple transition functions. Finally, the topology of the network is encoded in the definition of the transition functions which map a given state of the machine into another. A more in depth discussion of these properties can be found in the original paper~\cite{umm}.

In practice, the transition function is made dynamic by introducing non-linear passive elements such as memristors (memory resistors) and active elements such as  voltage-controlled voltage generators (VCVGs) into the circuits~\cite{dmm2}. Those elements provide ``memory", or hysteresis properties, to the system as well as dynamical control. Memory is an integral characteristic of the UMM computing model, which allows for time dependency of the differential equations governing the evolution of the system, thus a dynamic transition function. 

\subsection{Liquid-State Machines}

Informally, we can think of the LSM as a reservoir of water~\cite{liquid_2}. The process of sending the input signals into the machine is analogous to us dropping stones into the water, and the evolution of the internal states of the machine is the propagation of the water waves. Different waveforms will be excited depending on what stones were being dropped, how and when they were dropped, and the properties of the waveforms will encompass the information of the stones being dropped. Therefore, we can train a function that maps the waveforms to the corresponding output states that we desire.

Formally, we can define the machine using a set of filters and a trained function \cite{liquid_1}. A series of filters defines the evolution of the internal states, and the trained function maps the internal states to some output.

The set of filters must satisfy the point-wise separation property. This is defined as follows:
\begin{definition}
A class $B$ of basis filters has the \textnormal{pointwise separation property} if for any two input functions $u$ and $v$, with $u(s)\neq v(s)$ for some $s\leq t$, there is a basis filter $b\in B$ such that $(b\circ u)(t)\neq (b\circ v)(t)$. 
\end{definition}
This means that we can choose a series of filters such that the evolution of the internal states will be unique to any given signal, at any given time. In other words, this property ensures that different "stones" excite different "waveforms".\\

The trained output function must satisfy the ``fading memory'' property. This is defined as follows:
\begin{definition}
$F:U\rightarrow\mathbb{R}^n$ has \textnormal{fading memory} if for every internal state $u\in U$ and every $\epsilon>0$, there exist $\delta>0$ and $T>0$ so that $|(Fv)(0)-(Fu)(0)|<\epsilon$ for all $v\in U$ with $|u(t)-v(t)|<\delta$ for all $t\in[-T,0]$.
\end{definition}
Intuitively, this means that we do not need to know the evolution of the internal states from the infinite past to determine a unique output. Instead, we only need to know the evolution starting from a finite time $-T$. 

Since the liquid-state machine does not have any circuits that are hard-wired to perform specific tasks, it is particularly suited as a possible computational model for implementing learning algorithms. It is different from the modern neural networks in that it is not required for the reservoir (or hidden-layers) to be trained. The connections between the neurons are initialized randomly, so it is in essence a random projection from the input signals to some high-dimensional space, more specifically a pattern of activations in the network nodes, which is then read out by linear discriminant units. This effectively achieves a non-linear mapping from the inputs to outputs, with very little control over the process, and for this reason, it is not as favored compared to other neural networks. 

\subsection{Quantum Computers}
\label{qc}

There are many ways in which one can define a quantum computer. For simplicity sake, we consider the most general model of quantum computing and ignore its specific operations \cite{QI_bible}.

Consider a quantum computer with $n$ identical qubits, and each qubit can be expressed as a linear combination of $m$ basis states. The choice of the basis states can be arbitrary. However, they have to span the entire Hilbert space of the system. If we look at one single qubit, then every state that it admits can be expressed in terms of a linear combination of the $m$ basis states. (Typically $m$ is chosen equal to 2, but we do not restrict the quantum machine to this 
basis number here.) In other words, $\ket{\psi}=\sum_{i=0}^{m-1}c_i \ket{i}$, where $i$ simply labels the basis state, and $c_i$ is some complex number. Note that we have to impose the normalization condition $\sum_{i=0}^{m-1}|c_i|^2=1$.

Now, let us consider the whole system. In general, the total state can be expressed as $\ket{\psi}=\sum_{i_1=0}^{m-1}\sum_{i_2=0}^{m-1}...\sum_{i_n=0}^{m-1}c_{i_1 i_2...i_n}\ket{i_1 i_2...i_n}$. Here, $i_j$ denotes that the $j$-th qubit is in the $i$-th basis state. A basis state of the total wavefunction can be expressed as the tensor product of the basis states of the individual qubits. Then it is not hard to see that the total state will have a total number of $m^n$ basis states. Each basis state is associated with some complex factor $c_{i_1 i_2...i_n}$ where, again, the normalization condition $\sum_{i_1=0}^{m-1}\sum_{i_2=0}^{m-1}...\sum_{i_n=0}^{m-1}|c_{i_1 i_2...i_n}|^2=1$ has to be imposed.

Any quantum algorithm can be expressed as a series of unitary operations \cite{QI_bible}. Since the product of multiple unitary operations is still a unitary operation, we can then express the total operation after time $t$ with a one-parameter operator $\hat{U}(t)$, so that $\ket{\psi(t)}=\hat{U}(t)\ket{\psi(0)}$. In other words, the state of the quantum system at any point in time can be expressed as some unitary operation on the initial state. Note that $\hat{U}(t)$ can be either continuous or discrete in time. Either way, $\hat{U}(t)$ can be considered as the ``transition function" of the quantum computer.

Finally, we have to make measurements on the system in order to convert the quantum state into some meaningful {\it output} for the observable we are interested in. The process of measurement can be considered as finding the expectation value of some observable $\hat{O}$, so the {\it output function} of a quantum computer can be written as $\braket{\psi(t)|\hat{O}|\psi(t)}=\braket{\psi(0)|\hat{U}(t)^{\dagger}\hat{O} \hat{U}(t)|\psi(0)}$. Of course, to obtain an accurate result of the expectation value, many measurements have to be made. In fact, the initial state has to be prepared multiple times, evolved multiple times, and the 
corresponding expectation value at a given time, must be measured multiple times. A quantum computer is thus an intrinsically probabilistic-type of machine. 

\section{Mathematical Tools}\label{mathtools}

We now introduce the necessary mathematical tools that will allow us to construct a general description of a computing machine using set theory and cardinality arguments. Most of the definitions and theorems in this section, with their detailed proofs can be found in the literature on the subject (see, e.g., the textbook~\cite{cantor_1}). Note also that one does not have to understand the proofs in order to understand the theorems. Neither do the main results of this paper rely on understanding the proofs. Therefore, some of the proofs can be skipped over according to the interest of the reader.

First, we introduce the famous Cantor's Theorem \cite{cantor_2}:

\begin{theorem} 
The power set of the set of natural numbers has the same cardinaltiy of the set of real numbers, or $|\mathbb{R}|=2^{|\mathbb{N}|}$.
\end{theorem}

Here, ``cardinality" simply means the ``number" of elements in the set, or the ``size" of the set. Furthermore, one can also show the following theorem to be true: 

\begin{theorem}
Any open interval of real numbers has cardinality $|\mathbb{R}|$.
\end{theorem}

\begin{proof}
Note that the function $f(x)=\frac{x}{1-x^2}$ is a bijection from $(-1,1)$ to $\mathbb{R}$. Furthermore, there is a bijection from $(-1,1)$ to any finite open interval $(a,b)$, through the linear function $f(x)=-1+2\frac{x-a}{b-a}$. Therefore, there is a bijection from $(a,b)$ to $\mathbb{R}$, so the two sets have the same cardinality.
\end{proof}

Note that what we have just done is relating two infinities through exponentiation. We can further generalize these relations by introducing Beth numbers \cite{beth}, defined as follows:

\begin{definition}
Let $\beth_0=|\mathbb{N}|$, and $\beth_{\alpha+1}=2^{\beth_{\alpha}}$ for all $\alpha\in \mathbb{N}$.
\end{definition}

By this definition, we see that $\beth_1=2^{\beth_0}=2^{|\mathbb{N}|}=|\mathbb{R}|$. Of course, each Beth number is strictly greater than the one preceding it. 

The following theorem \cite{cantor_1} allows us to perform arithmetics on infinite cardinal numbers, and derive relationships between Beth numbers. 

\begin{theorem}
\label{max}
Given any two positive numbers, $\mu$ and $\kappa$, if at least one of them is infinite, then $\mu+\kappa=\mu\kappa=\max\{\mu,\kappa\}$.
\end{theorem}

Using this theorem, we can prove the following properties of Beth numbers:

\begin{corollary}
\label{trivial}
The addition or multiplication of two Beth numbers equals to the greater of the two, or $\beth_\beta\beth_\alpha=\beth_\beta+\beth_\alpha=\beth_\beta$ if $\alpha \leq\beta$ for all $\alpha, \beta\in\mathbb{N}$.
\end{corollary}

\begin{corollary}
\label{ari}
$\mu^{\beth_{\alpha}}=\beth_{\alpha+1}$ if $2\leq\mu\leq\beth_{\alpha+1}$. $(\beth_{\alpha})^{\kappa}=\beth_{\alpha}$ if $1\leq\kappa\leq\beth_{\alpha-1}$
\end{corollary}

\begin{proof}
Corollary \ref{trivial} comes directly from Theorem \ref{max} and the fact that each Beth number is greater than its predecessors.
\end{proof}

\begin{proof}
Corollary \ref{ari} can be proven as follows. By definition $2^{\beth_{\alpha}}=\beth_{\alpha+1}$. Furthermore, $(\beth_{\alpha+1})^{\beth_{\alpha}}=(2^{\beth_{\alpha}})^{\beth_{\alpha}}=2^{(\beth_{\alpha}\beth_{\alpha})}=2^{\beth_{\alpha}}=\beth_{\alpha+1}$, where $\beth_{\alpha}\beth_{\alpha}=\beth_{\alpha}$ from Corollary \ref{trivial}. Since $2\leq\mu\leq\beth_{\alpha+1}$, then $\beth_{\alpha+1}= 2^{\beth_{\alpha}}\leq\mu^{\beth_{\alpha}}\leq(\beth_{\alpha+1})^{\beth_{\alpha}}=\beth_{\alpha+1}$. This implies $\mu^{\beth_{\alpha}}=\beth_{\alpha+1}$. The second equality can be proven in the same vein.
\end{proof}

In the following section, we will define computing models using Cartesian products and mapping functions. The following two theorems will be helpful \cite{cantor_1}:

\begin{theorem}
Let the set $S$ be the Cartesian product of the sets $S_1$ and $S_2$, or $S=S_1\times S_2$. Then $|S|=|S_1||S_2|$.
\end{theorem}

\begin{theorem}
\label{map}
Let $f:S\rightarrow T$ be a function that maps set $S$ to set $T$, and let $F$ be the set of all possible functions $f$. Then $|F|=|T|^{|S|}$.
\end{theorem}

Finally, we introduce the following theorem, which can be derived directly from the definition of cardinality \cite{cantor_1}:

\begin{theorem}
\label{card}
Two sets have the same cardinality if and only if there is a bijection between them. 
\end{theorem}

An interesting corollary of this theorem is the following:

\begin{corollary}
There is a bijection between $\mathbb{R}^n$ and $\mathbb{R}^m$, for any $n,m\in\mathbb{N}$.
\end{corollary}

\begin{proof}
From Corollary \ref{ari}, we see that $|\mathbb{R}^n|=|\mathbb{R}|^n=(\beth_1)^n=\beth_1$. Similarly, $|\mathbb{R}^m|=\beth_1$. Therefore, the two sets have the same cardinality, so there is a bijection between the two. 
\end{proof}

Note that for complex coordinate spaces, $|\mathbb{C}^n|=|\mathbb{C}|^n=|\mathbb{R}^2|^n=|\mathbb{R}|^{2n}=(\beth_1)^{2n}=\beth_1$, so this bijection further extends to complex coordinate spaces.

\section{General Computing Model}\label{setcomp}

We have now introduced all the mathematical tools necessary to describe a computing machine under a set-theoretical framework. There are two sets that we are interested in: the set of all {\it internal states} that the machine supports, and the set of all realizable {\it transition functions}.

\subsection{Internal States}
\label{int}

Consider a general computing machine. We define a state variable $s$ to describe the full internal state of the machine, and we define the set $S$ as the collection of all possible states. The internal state should encompass all the necessary information of the machine at a given iteration. Formally, we can define it as follows:

\begin{definition}
The symbol $s$ is the internal state of the machine if and only if given $s$, the machine states of subsequent iterations are all fully determined for any algorithm.
\end{definition} 

For example, consider the full internal state of a Turing machine. The internal state should consist of three sub-states: the register state of the control ($s^{(1)}$), the tape symbols written on the cells ($s^{(2)}$), and the current address of the head ($s^{(3)}$) \cite{trans}. Given these three substates and some transition function (depending on how the machine is coded), then the processor will know what to write (thereby changing $s^{(2)}$), in what direction to move (thereby changing $s^{(3)}$), and what new register state to be in (thus changing $s^{(1)}$). Therefore, we see that the next state is uniquely determined, so the transition function is well-defined under this framework. Note that if we eliminate any of the three sub-states, then it will be impossible for us to uniquely determine the internal state of the subsequent iteration.

Under the set theoretical framework, we can express the set of internal states as the Cartesian product of the three sets of the respective sub-states, or $S=S^{(1)}\times\ S^{(2)} \times S^{(3)}$, then $|S|=|S^{(1)}||S^{(2)}||S^{(3)}|$. By calculating the cardinality of the set $S$, we essentially know the ``number" of internal states that a general Turing machine can support. 

Consider a Turing machine with $m$ tape symbols, $n$ tape cells, and $k$ register states. It is easy to see that $|S^{(1)}|=k$, $|S^{(2)}|=m^n$, and $|S^{(3)}|=n$. Therefore, we can calculate $|S|=|S^{(1)}||S^{(2)}||S^{(3)}|=km^nn<|\mathbb{N}|=\beth_0$. The strict inequality comes from the fact that $m$, $n$, and $k$ are all finite numbers. In other words, this is a finite digital/discrete machine, and in general, the inequality $|S|<\beth_0$ is true for a finite digital machine.

On the other hand, if we consider the theoretical model of a Turing machine with infinite tape cells $\mu=|\mathbb{N}|=\beth_0$, then the calculation changes significantly. In this case, we have $|S^{(1)}|=k$, $|S^{(2)}|=m^{\beth_0}=\beth_1$ (from Corollary \ref{ari}), and $|S^{(3)}|=\beth_0$. Therefore, we see that $|S|=|S^{(1)}|S^{(2)}||S^{(3)}|=k\beth_1\beth_0=\beth_1$ (from Corollary \ref{trivial}). For the purpose of proving Turing-completeness, we use this model of infinite tape.

\subsection{Transition Function}
\label{trans func}

\subsubsection{General Definition of Transition Function}

The operation of any computing machine is defined using a transition function \cite{trans} (or a set of transition functions such as in a general UMM \cite{umm}). In this paper, we only consider deterministic machines, or well-defined transition functions. Here, we give a much more general definition of the transition function. Essentially, we are throwing away constraints such as initial states, accepting states, tape symbols, and so forth, and simply define the transition function as a mapping from some state to some other (or the same) state. We can then formally define the transition function as follows:

\begin{definition}
\label{trans}
Let $S$ be any set, then $\varphi_S$ is said to be a \textnormal{transition function} on $S$ if, for every $s\in S$, we have a unique $\varphi_S(s)\in S$. In other words, $\varphi_S$ is a function that maps $S$ to a subset of itself, or $\varphi_S: S\mapsto S'$, where $S'\subseteq S$. We denote the set of all possible transition functions on $S$ as $\Phi_S$. 
\end{definition}

\begin{figure}
\includegraphics[scale=0.65]{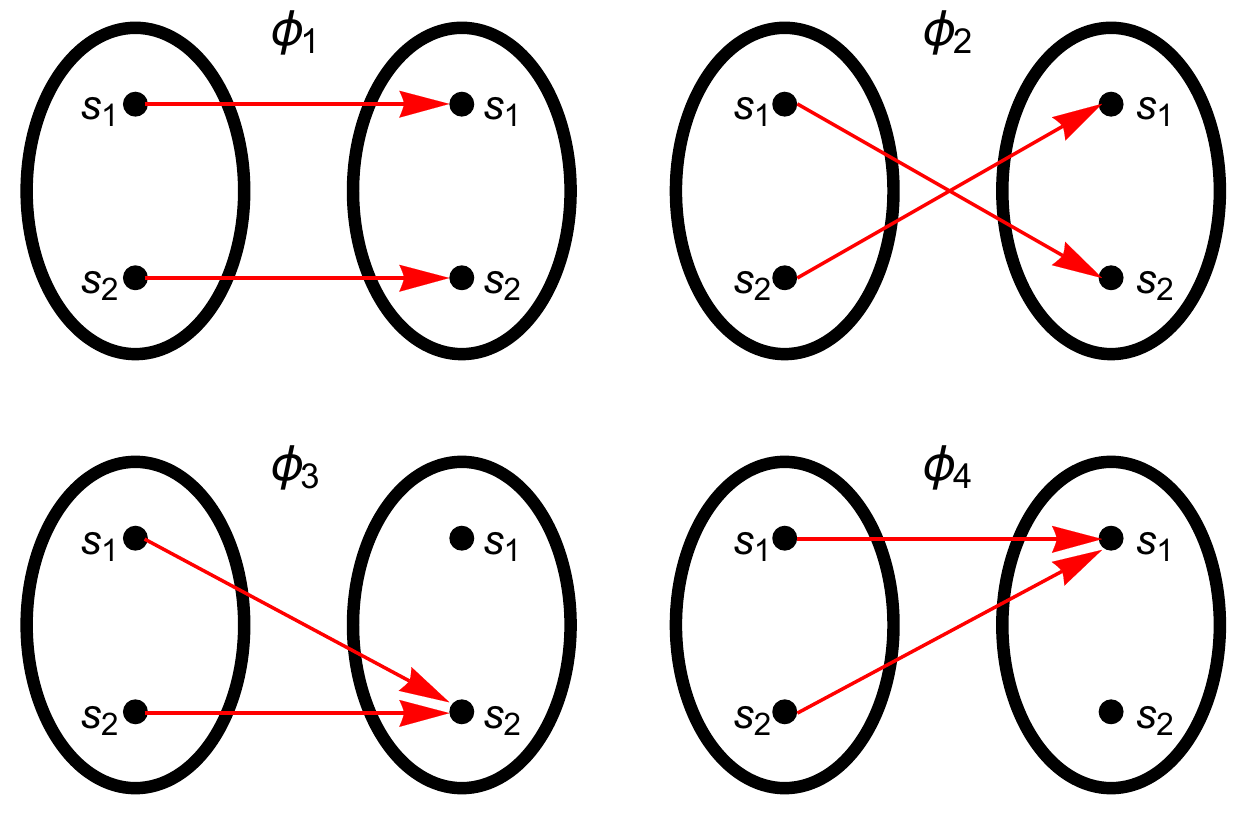}
\caption{The set of all transition functions for $|S|=2$. There are $2^2=4$ transition functions in total. A transition function on this set is described by two red arrows.}
\label{trans_draw}
\end{figure}

From Theorem \ref{map}, it is easy to see that the cardinality of $\Phi_S$ is simply $|\Phi_S|=|S|^{|S|}$. It is important to note that we are not defining the transition function based on \textit{the operation of the machine}. Instead, we are defining the transition function as the \textit{mapping of a set to itself} (see Fig.~\ref{trans_draw} for a schematic representation). As a result, there will be some transition functions that are not \textit{realizable} in the context of a particular computational model. The next subsection, section \ref{turing_example}, gives illustrative examples of transition functions that cannot be realized by a Turing machine.

\subsubsection{Turing Machine as Example}\label{turing_example}

For a Turing machine, after the machine is coded, the transition function remains stationary and cannot be changed during the execution of an algorithm (unlike a UMM where the transition function is dynamic). In other words, the machine will take some initial internal state $s_i$ and apply some transition function $\varphi_S$ to it recursively until the final state $s_f$ is reached and the machine halts. We can express this as 
\begin{equation}
s_f=\underbrace{\varphi_S(...\varphi_S(\varphi_S}_\text{n iterations}(s_i))).
\end{equation}
Furthermore, when the final state is reached, we should have $s_f=\varphi_S(s_f)$, meaning that the transition function should not alter the final state, and this represents the termination of the algorithm.

It is obvious that for any Turing machine algorithm, one can always find a transition function associated with it; however, the reverse is not true. In fact, there are transition functions in $\Phi_S$ that are not \textit{realizable} on a Turing machine. Some examples are writing symbols on two cells at once, writing a symbol on a cell not directly under the head, moving the head two to the right/left of the current one, and so on. However, note that these operations are not excluded \textit{a priori} from the full set of transition functions.

\subsection{Realizable Transition Functions}
\label{realizable}

For convenience, let us denote the set of all realizable transition functions acting on set $S$ as $\Phi'_S$. Then of course, $\Phi'_S\subseteq\Phi_S$ must be true. We call $\Phi'_S$ the \textit{realizable} set of transition functions, and $\Phi_S$ the \textit{full} set of transition functions. 

As a simple example, consider a machine with only two states, 0 and 1. Furthermore, the machine can only perform two functions at each iteration. The first function is to do nothing at all, and the second function is to switch between the two states. It is clear that the set of internal states can be expressed as $S=\{0,1\}$. This machine only supports two transition functions, which can be expressed as $\varphi_1(s)=s$ and $\varphi_2(s)=1-s$ for all $s\in S$. 

As a side note, we could have also written the second transition function as $\varphi_S(s)\equiv 1+s \Mod{2}$. However, since the transition function is defined by its mapping rather than its form, the 2 different expressions represent the same transition function. In this paper, we do not concern ourselves with the \textit{representation} of the transition functions.

Therefore, we see that the \textit{realizable} set of transition functions is simply a set with two elements, $\Phi'_S=\{\varphi_1, \varphi_2\}$. Note that there are two more transition functions possible, but they are not supported by the machine. They are $\varphi_3(s)=0$, and $\varphi_4(s)=1$. However, they are not excluded from the \textit{full} set of transition functions, $\Phi_S=\{\varphi_1,\varphi_2,\varphi_3,\varphi_4\}$.

Given $S$ and $\Phi'_S$, one can fully describe the machine \textit{structure}, so we can naturally make the following definition:

\begin{definition}
A machine $S\times\Phi'_S$ is a machine with the set of internal states $S$ and the realizable set of transition functions $\Phi'_S$
\end{definition} 

Here, the product notation is borrowed from \textit{group action} \cite{dummit}, though it is somewhat of an ``abuse of notation'' because $\Phi'_S$ is not necessarily a group. Nevertheless, this does not matter for what we are about to discuss next. 

\subsection{Isomorphism}
\label{prop}

\begin{figure}
\includegraphics[scale=0.65]{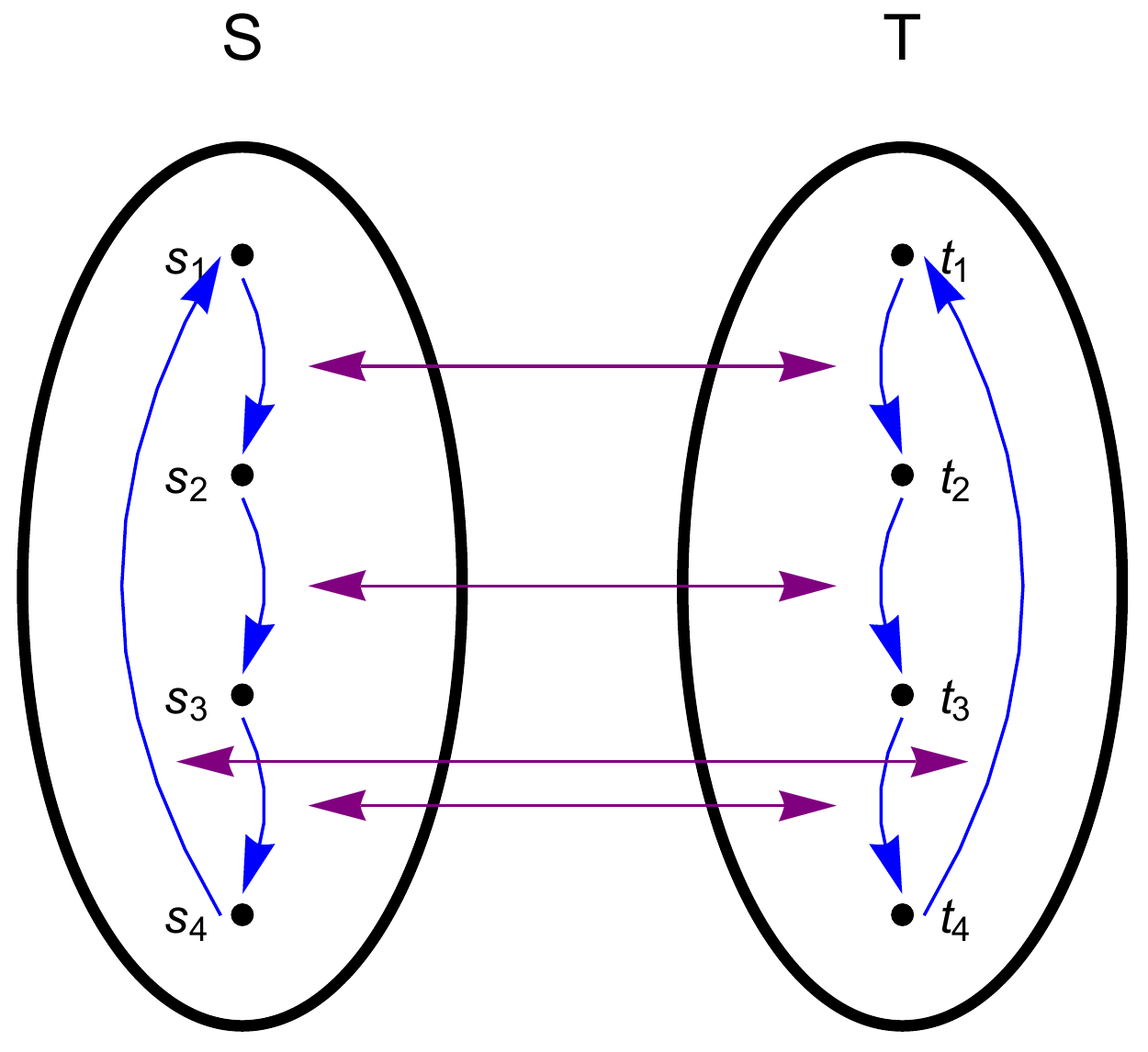}
\caption{Let $S$ and $T$ be two sets of the same cardinality, and label the set elements such that $s_i$ pairs with $t_i$. The collection of all blue arrows within a set represent a single transition function for each set. Informally, the pairing of the two transition functions is represented by the purple arrows that pair the arrows in set $S$ with their respective counterparts in set $T$.}
\label{bi_draw}
\end{figure}

Traditionally, the notion of ``equivalence" is defined as follows \cite{complete_1}:
\begin{definition}
Two machines, $A$ and $B$, are said to be equivalent if and only if $A$ can simulate $B$, and $B$ can simulate $A$.
\end{definition}

If machines $A$ and $B$ are equivalent, then machine $A$ can simulate the machine processes of machine $B$, and vice versa \cite{complete_1}. But recall that machine processes can be expressed in terms of transition functions acting on internal states. So informally, two machines are equivalent if and only if they have the same internal states and transition functions.

Again, borrowing from group theory the notion of \textit{isomorphism} \cite{dummit}, we can formalize what it means exactly for two sets to be the ``same":

\begin{definition}
Two machines, $S\times\Phi'_S$ and $T\times\Phi'_T$, are isomorphic if and only if the following holds:

1. $S$ and $T$ are of the same cardinality.

2. $\Phi'_S$ and $\Phi'_T$ are of the same cardinality.

3. There exist bijections $g:S\mapsto T$ and $h:\Phi'_S\mapsto\Phi'_T$ such that $g(\varphi_S(s))=h(\varphi_S)(g(s))$ for any $s\in S$ and $\varphi_S\in\Phi'_S$.
\end{definition}

At this point, we can use the terminology \textit{equivalence} and \textit{isomorphism} interchangeably. From the definition of \textit{isomorphism}, we can derive directly the following theorem about machines with \textit{full} sets of transition functions:

\begin{theorem}
\label{full biject}
If $|S|=|T|$, the machine $S\times\Phi_S$ is isomorphic to $T\times\Phi_T$, where $\Phi_S$ and $\Phi_T$ are full sets of transition functions acting on their respective sets.
\end{theorem}

\begin{proof} We show that the two machines satisfy the three conditions for isomorphism:

1. Since $|S|=|T|$, then $S$ and $T$ are of the same cardinality.

2. Since $|\Phi_S|=|S|^{|S|}=|T|^{|T|}=|\Phi_T|$, then $\Phi_S$ and $\Phi_T$ are of the same cardinality.

3. Let $g:S\mapsto T$ be any arbitrary bijection. Then we define $h(\varphi_S)=g \varphi_S g^{-1}$. It is obvious that $h$ is in fact a bijection under this definition. Furthermore, for any $s\in S$ and $\varphi_S\in\Phi'_S$, we have $h(\varphi_S)(g(s))=g\varphi_Sg^{-1}g(s)=g(\varphi_S(s))$.

\end{proof}

It is important to note that this theorem is only useful if the two machines under comparison both support their \textit{full} sets of transition functions. If two machines only support \textit{realizable} sets of transition functions, then they are not necessarily isomorphic even if they have the same ``size". For example, consider two machines both with the same internal states, $S=\{0,1\}$. Machine $A$ only supports the function $\varphi_A=s$, and machine $B$ only supports the function $\varphi_B=1-s$. Informally, machine $A$ is a machine that does nothing, and machine $B$ is a machine that keeps switching between the two states. It is clear that there is no appropriate mapping between the two machines. 

However, note that it is possible for two machines to be \textit{isomorphic} without supporting the \textit{full} set of transition functions. For example, consider again two machines both with the same internal states, $S=\{0,1\}$. Machine $A$ only supports the function $\varphi_A=0$, and machine $B$ only supports the function $\varphi_B=1$. At first site, the two machines may seem different. However, we can map the 0 state of machine $A$ to the 1 state of machine $B$, and the 1 state of machine $A$ to the 0 state of machine $B$. And if we denote this mapping as $f(s)=1-s$, then we can easily show that $\varphi_B(f(0))=f(\varphi_A(0))=1$ and $\varphi_B(f(1))=f(\varphi_A(1))=1$. Therefore, we see that the two machines are isomorphic. Informally, both machines work by bringing every state to some fixed state, so they \textit{behave} the same even though they have different sets of transition functions. 

\textit{Isomorphism} is an equivalence relation. So if machine $A$ is isomorphic to machine $B$, then we can denote this as $A\equiv B$. Then, the following properties clearly hold: 
\begin{enumerate}
\item $A\equiv A$.
\item if $A\equiv B$, then $B\equiv A$
\item if $A\equiv B$ and $B\equiv C$, then $A\equiv C$.
\end{enumerate}

\subsection{Sub-machines}

Traditionally, the notion of ``completeness" is defined as follows \cite{complete_1}:
\begin{definition}
Machine $A$ is $B$-\textnormal{complete} if and only if the former can simulate the latter.
\end{definition}

The difference of this definition from the definition of \textit{equivalence} is that the reverse need not be true. In other words, it is not required that $B$ can simulate $A$. In a sense, we can say that machine $A$ is a more ``general" computation model than machine $B$. 

Informally, if machine $A$ is $B$-complete, then the internal states and transition functions of machine $B$ are ``included" in those of machine $A$. To formalize the concept of ``inclusion", we introduce the notion of a \textit{sub-machine}. Then we can define \textit{sub-machine} in a way such that $A$ is $B$-complete if and only if we can find a sub-machine of $A$ such that it is isomorphic to $B$. 

The most obvious way to construct a \textit{sub-machine} from a given machine is through the \textit{reduction} of sets - the reduction of the internal state set, the reduction of the transition function set, and the reduction of both. First, let us look at the reduction of the transition function set:

\begin{definition}
Consider two machines, $S\times\Phi'_1$ and $S\times\Phi'_2$. We define the latter to be a \textnormal{functional reduction} of the former if and only if $\Phi'_2\subseteq\Phi'_1$.
\end{definition}

Note that the subscript $S$ of $\Phi_S$ is implied. Essentially, we are keeping $S$ the same while shrinking the set $\Phi_S$. As an example, consider again the machine with internal state set $S=\{0,1\}$ and transition functions $\varphi_1(s)=s$ and $\varphi_2(s)=1-s$. We can \textit{reduce} the transition function set by throwing away $\varphi_2$ and keeping only $\varphi_1$. Originally, the machine can perform two functions, doing nothing and switching. After the functional reduction, the machine loses its switching function, and can only perform the trivial function of doing nothing. 

There is also the reduction of the internal state set, which we can define as follows:

\begin{definition}
Consider two machines, $S\times\Phi'_S$ and $S'\times\Phi'_{S'}$. We define the latter to be a \textnormal{state reduction} of the former if and only if the following holds:

1. $S'\subseteq S$.

2. For any $\varphi'\in\Phi'_{S'}$, there exists a $\varphi\in\Phi'_S$ such that $\varphi'(s)=\varphi(s)$ for every $s\in S'$.

3. For any $\varphi\in\Phi'_S$  with image $\varphi(S')\subseteq S'$, there exists a $\varphi'\in\Phi'_{S'}$ such that $\varphi'(s)=\varphi(s)$ for every $s\in S'$.
\end{definition}

Note that after a state reduction, both the internal state set and the transition function set change. Furthermore, note that a state reduction is unique given some $S'\subseteq S$. The intuition behind this form of reduction is fairly straightforward. Informally, we simply throw away some elements from the internal state set, and after that, we throw away the functions that cannot be \textit{well-defined} without the state elements that we have thrown out. 

It is important to note that the state reduction of a machine with a full  set of transition functions is still a machine with a full set of transition functions, or compactly stated:

\begin{theorem}
\label{full_reduce}
The state reduction of $S\times\Phi_S$ is $S'\times\Phi_{S'}$, where $S'\subseteq S$ and $\Phi$ denotes the full set of transition functions. 
\end{theorem}

This follows directly from the definition of sub-machine. 

It is obvious that a machine can simulate a reduced version of itself. So at this point, we can appropriately define a \textit{sub-machine} as a series of state and functional reductions of the original machine. However, we can simplify this definition if we first introduce a few lemmas:

\begin{lemma}
\label{function}
Applying functional reductions twice is equivalent to applying one functional reduction.
\end{lemma}

\begin{lemma}
\label{state}
Applying state reductions twice is equivalent to applying one state reduction.
\end{lemma}

\begin{lemma}
\label{state_function}
A reduction is expressible in terms of a functional reduction followed by a state reduction if and only if it is expressible in terms of a state reduction followed by a functional reduction.
\end{lemma}

\begin{proof}
The first two lemmas are fairly straightforward to prove, so let us focus on the third lemma. Let $\Phi'_{S'}$ be a state reduction of $\Phi'_S$, then we can define a surjection $h:\Phi'_S\mapsto\Phi'_{S'}$ such that $(h(\varphi))(s)=\varphi(s)$ for every $s\in{S'}$ and every $\varphi\in\Phi'_S$. If we let $\Phi''_{S}\subseteq\Phi'_{S}$ be a functional reduction, then it is obvious that $h(\Phi''_{S})\subseteq\Phi'_{S'}$, meaning that $h(\Phi''_{S})$ is a functional reduction of $\Phi'_{S'}$. In other words, the state reduction of a functional reduction is also some functional reduction of a state reduction. The converse can be proven similarly by inverting the mapping so that we have some multivalued function $h^{-1}:\Phi'_{S'}\mapsto\Phi'_S$. In this case, if we let $\Phi''_{S'}\subseteq\Phi'_{S'}$, then we have $h^{-1}(\Phi''_{S'})\subseteq\Phi'_S$. Note that $h^{-1}$ is not really a function in the conventional sense, but this does not matter for the proof.  

\end{proof}

Note that Lemma \ref{function} implies that multiple functional reductions can be ``compressed" into one, and Lemma \ref{state} implies that multiple state reductions can also be ``compressed" into one. Furthermore, Lemma \ref{state_function} implies that the order in which we apply the functional and state reductions does not matter. With these in mind, we can finally make the following definition of \textit{sub-machine}:

\begin{definition}
Machine $B$ is a \textnormal{sub-machine} of $A$ if and only if $B$ can be expressed as a functional reduction followed by a state reduction of $A$.
\end{definition}

To see that this definition is appropriate, recall that a sub-machine can be constructed through a series of reductions. Since by Lemma \ref{state_function}, the order of functional and state reductions does not matter, we can choose to apply all the functional reductions first, and then apply all the state reductions. Furthermore, from Lemma \ref{function} and \ref{state}, we can condense all the functional reductions into one, and all the state reductions into one as well. Therefore, in the end, we are left with one functional reduction followed by one state reduction. 

An important thing to note is that \textit{sub-machine} is an order relation. So if machine $B$ is a sub-machine of machine $A$, then we can denote this as $B\preceq A$. The following properties then hold:
\begin{enumerate}
\item $A\preceq A$.
\item If $A\preceq B$ and $B\preceq A$, then $A\equiv B$.
\item If $B\preceq A$ and $C\preceq B$, then $C\preceq A$.
\end{enumerate}

At this point, we are in the position to formalize the relation of \textit{completeness} between two machines under our set-theoretic framework. We can define \textit{completeness} through the use of \textit{sub-machine}:

\begin{definition}
Machine $S\times\Phi'_S$ is $T\times\Phi'_T$-complete if we can find a sub-machine of $S\times\Phi'_S$ such that it is isomorphic to $T\times\Phi'_T$.
\end{definition}

Intuitively, what this means is that we are able to find a mapping between the two machines such that the machine processes of $T\times\Phi'_T$ are ``included" in the machine processes of $S\times\Phi'_S$. In other words, machine $S\times\Phi'_S$ can simulate $T\times\Phi'_T$. Under this definition, we can derive an important result about machines with full sets of transition functions:

\begin{theorem}
\label{master}
Machine $S\times\Phi_S$ is $T\times\Phi'_T$-complete if $\Phi_S$ is the full set of transition functions on $S$ and $|T|\leq |S|$.
\end{theorem}

\begin{proof}
Let $|S'|\subseteq |S|$ be some arbitrary subset of $S$ such that $|S'|=|T|$, then we can define some arbitrary bijection $g:T\mapsto S'$. Furthermore, we can define an injection $h:\Phi'_T\mapsto \Phi_{S'}$ such that $h(\varphi_T)=g\varphi_T g^{-1}$ for every $\varphi_T\in \Phi'_T$. To show that $h$ is in fact an injection, note that:
\begin{enumerate}
\item Given any $\varphi_T\in\Phi'_T$, we have $h(\varphi_T)s_1=g\varphi_Tg^{-1}s_1=g\varphi_Tt_1=gt_2=s_2\in S'$ for any $s_1\in S'$. Therefore, $h(\varphi_T):S'\mapsto S'$ is a transition function on $S'$, or $h(\varphi_T)\in\Phi_{S'}$.
\item Furthermore, $s_2=h(\varphi_T)s_1$ is uniquely determined for every $s_1\in S'$. So there is only one $h(\varphi_T)$ possible for any given $\varphi_T\in\Phi'_T$. Therefore, $h$ is a well-defined mapping.
\item Assume that $\varphi_T\neq\varphi_T'$, then there exists $t_1\in T$ such that $\varphi_T(t_1)=t_2\neq t_3=\varphi'_T(t_1)$. So $h(\varphi_T)s_1=s_2\neq s_3=h(\varphi'_T)s_1$, implying that $h(\varphi_T)\neq h(\varphi'_T)$. Therefore, $h$ is an one-to-one mapping.
\end{enumerate}
Now, if we let $\Phi'_{S'}\subseteq\Phi_{S'}$ be the set of all transition functions $\varphi_{S'}$ with a well-defined inverse $h^{-1}(\varphi_{S'})$, then it is easy to see that $h:\Phi'_T\mapsto\Phi'_{S'}$ is a bijection. Furthermore, we see that machine $T\times\Phi'_T$ is isomorphic to machine $S'\times\Phi'_{S'}$. But $S'\times\Phi'_{S'}$ is a functional reduction of $S'\times\Phi_{S'}$, which itself is a state reduction of $S\times\Phi_S$, so $S'\times\Phi'_{S'}$ is a sub-machine of $S\times\Phi_S$. Putting everything together, we see that $T\times\Phi'_T$ is isomorphic to a sub-machine of $S\times\Phi_S$, so $S\times\Phi_S$ is $T\times\Phi'_T$-complete.
\end{proof}

This theorem essentially states that if a machine supports the full set of transition functions (such as the UMM), then it can simulate any machine of a smaller size. This will be our basis for showing completeness relations for the UMM with respect to the other computation models.

\section{Revised Machine Definitions within Set Theory}\label{SET-rewrite}

Up to this point, we have avoided the discussion of the concept of ``output states" of a machine. However, under this new set-theoretic framework, this concept is easily expressible.

In general, the internal state of the machine must be ``decoded'' into an output state to be read by the user. We can denote the set of all possible output states as $F$. Then it is obvious that $|F|\leq |S|$, otherwise we would not be able to find a function that maps $S$ to $F$. In other words, every internal state must correspond to some output state, and it is easy to show that the output function is expressible as a transition function.

To show this, we first choose a subset $S_F\subseteq S$ such that $|S_F|=|F|$, then there is a bijection between $F$ and $S_F$. Therefore, we can simply describe the output mapping function as a transition function that maps $S$ to $S_F$. It is then clear that the set of all possible functions mapping internal states to outputs must be a subset of $\Phi_{S}$, so we can simply expand the \textit{realizable} set $\Phi'_S$ to include these functions, without the need to redefine a new set.

The mathematical framework has now been fully established, and we are ready to redefine the three machines we are considering in this work within this 
framework.

\subsection{Universal Memcomputing Machines}

In the original definition of the UMM, there is the complication of input and ouput pointers (see Eq.~(\ref{UMMdef})). Within the new set-theory framework, we can avoid 
the concept of pointers, since the transition function always reads the full internal state (all the cells) and writes the full internal state. In other words, we consider the combination of all cell states as a whole, and make no effort in describing which cell admits which state. In this case, intrinsic parallelism is obviously implied.

Let us then discuss the cardinality of the set of internal states. For a UMM, there is a finite number of cells, $n<\beth_0$, and each cell may admit a continuous state (with cardinality $\beth_1$). In this case, it is easy to show that $|S|=(\beth_1)^{n}=\beth_1$.

Furthermore, in the original definition of the UMM [Eq.~(\ref{UMMdef})], there are no constraints on the transition function $\delta$. This means that we can use the full set of transition functions, $\Phi_{S}$, to describe the machine. In this case, functional polymorphism is obviously implied. Therefore, we can define the UMM machine as $S\times\Phi_{S}$, with $|S|=\beth_1$. 

\subsection{Liquid-State Machine}

It is not hard to see that the internal state structures for the LSM and the UMM are similar. Instead of memory cells, the LSM has neurons. But if we make the conservative assumption that there are no constraints on the internal states of the LSM, then the cardinality of the set of internal states for this machine is the same as that of a UMM, $|S|=\beth_1$. However, what distinguishes the LSM from the UMM is that the set of transition functions for the LSM is not full. 

Recall that the LSM consists of a series of filters and an output function (see Sec.~\ref{definition}). The set of filters satisfies the point-wise separation property, and the function satisfies the fading-memory property. There is no need to express the two properties in the language of our new framework. Instead, it is enough to note that the point-wise property is a property of the transition function set, while the fading-memory property is a property of the transition function itself.

Therefore, we can find a subset $(\Phi_{S})_1\subseteq\Phi_{S}$ such that its elements represent the filters, with the subset itself satisfying the point-wise separation property. As discussed earlier, we can express the output function as a transition function, so we can find a subset $(\Phi_{S})_2\subseteq\Phi_{S}$ such that its elements represent the output functions, and they satisfy the fading-memory property. Then, we can take the union of the two subsets $\Phi'_{S}=(\Phi_{S})_1\cup (\Phi_{S})_2$ to get the realizable set of transition functions on $S$.

Therefore, we can describe the LSM as $S\times\Phi'_{S}$, with $|S|=\beth_1$. The specific structure of the machine can be defined by expressing the two properties as constraints on $\Phi'_{S}$. This is a slightly tedious process, so we will not be presenting it here, since it is irrelevant for our conclusions. 

\subsection{Quantum Computers}

Again, consider a quantum computer with $n$ identical qubits, each having $m$ basis states. In subsection \ref{qc}, we have shown that the total number of basis states for the entire system is $m^n<|\mathbb{N}|$. Each basis state is associated with some complex factor $c_{i_1 i_2...i_n}\in\mathbb{C}$. These factors are constrained by the normalization condition $\sum_{i_1=0}^{m-1}\sum_{i_2=0}^{m-1}...\sum_{i_n=0}^{m-1}|c_{i_1 i_2...i_n}|^2=1$. Furthermore, in practice we usually ignore an overall phase factor since it does not affect the expectation value of an observable.

At this point, it is clear that we can fully describe a quantum state as the Cartesian product of the complex factors for all basis states. In other words, we can represent an internal state as $c_{0_1 0_2 ... 0_n}\times c_{1_1 0_1 ... 0_n}...\times c_{(m-1)_1 (m-1)_2 ... (m-1)_n}$, where there are $m^n$ factors.

Given this information, we can calculate the cardinality of the full set of internal states to be $|S|=|\mathbb{C}^{m^n}|-|\mathbb{R}^2|=|\mathbb{C}^{m^n}|\leq|\mathbb{C}^{\mathbb{N}}|=|\mathbb{C}|^{|\mathbb{N}|}=(|\mathbb{R}|^{|\mathbb{N}|})^2=(\beth_1^{\beth_0})^2=(\beth_1)^2=\beth_1$. (The unimportant $|\mathbb{R}^2|$ is from the normalization condition and factoring out the overall phase factor.) In addition, $|S|=|\mathbb{C}^{m^n}|\geq|\mathbb{R}|=\beth_1$. Therefore, we have $|S|=\beth_1$.

The full set of internal states contains quantum states with varying degrees of entanglement. It is worth stressing though that, in practice, it is extremely hard to construct a quantum computer that can support the full set of quantum states. For example, it is very challenging to prepare 100 qubits that are fully entangled. (The current record is on the order of tens of fully-entangled qubits \cite{14_qubit}\cite{10_qubit}.) Therefore, the actual set $S'$ is a small subset of $S$, or $S'\subseteq S$, unless $m$ and $n$ are both very small. However, since we are making here only theoretical arguments, let us just assume that the full set $S$ of all possible entangled states can be supported.

As discussed previously, the transition functions of a quantum computer can be expressed as unitary operations on some initial state. The set of all unitary operations is obviously a strict subset of the full set of transition functions. For example, one cannot find an unitary operation that collapses every single state to $\ket{00...0}$ (setting $c_{0_1 0_2 ... 0_n}=1$, and setting all the other factors to 0), though this is excluded \textit{a priori} from $\Phi_S$. Therefore, we can describe the quantum computer as $S\times\Phi'_S$, where $\Phi'_S$ is the set of unitary evolutions governed by all possible Hamiltonians. Note that it is also possible for a quantum computer to have a dynamic transition function. This is the case if one performs \textit{quantum annealing} for finding the ground states \cite{anneal}. In this process, the Hamiltonian of the system is varied slowly, thus giving rise to a dynamic transition function.

A few words on the ``output function'' of a quantum computer are also in order. The output function essentially represents the operation of taking the expectation value of some observable on the internal state, or $\braket{\psi|\hat{O}|\psi}$. This maps the set of internal states to the set of output states $F\subseteq\mathbb{R}$ (expectation values have to be real), so we obviously have $|F|\leq|S|$. All things considered, a quantum computer can then be described as $S\times \Phi'_S$, where $|S|=\beth_1$, and $\Phi'_S$ is its realizable set of transition functions. 

\section{Universality of Memcomputing Machines}\label{universality}

From the above discussions, we can summarize all the results we have obtained so far, and express all the machines we considered here in their most general form:
\begin{itemize}
\item Turing Machine: $T\times\Phi'_{T}$, $|T|\leq\beth_1$,
\item Liquid-State Machine: $L\times\Phi'_{L}$, $|L|=\beth_1$,
\item Quantum Computer: $Q\times\Phi'_{Q}$, $|Q|=\beth_1$,
\item Universal Memcomputing Machine: $M\times\Phi_{M}$, $|M|=\beth_1$.
\end{itemize}

Therefore, by applying Theorem \ref{master}, we see that a UMM can simulate any Turing machine, any liquid-state machine, and any quantum computer. 

At this point, the goal of this paper is essentially accomplished. However, we can expand on this for each pair of machines separately. In particular, let us briefly discuss how a mapping between a UMM and the three other machines can be realized in theory. Of course, this mapping does not tell us anything about the resources required for a UMM to simulate these other machines. Hence, this is by no means a discussion on how to realize an efficient or practical mapping. 

\subsection{UMM vs. Turing Machines}

Let us look at the mapping between a Turing machine and a UMM. First, we map each tape cell to a memory cell (memcell). We can denote these memcells collectively as a ``memtape". The tape symbols can be mapped to the internal states of each memcell of the memtape. 

Then, we can map the state register to another memcell which we will denote as ``memregister". The state of the Turing machine is then stored as the internal state of the memregister. Finally, we can store the current address of the head as an internal state of yet another memcell which we denote as ``memaddress".

We can then wire the memcells together into a circuit such that it simulates the operation of the Turing machine. Note that as a result of functional polymorphism, we do not have to re-wire the circuit each time we choose to run a different algorithm. The circuit first reads the memregister and the memaddress, so that it knows which memcell of the memtape to modify and how to modify it. After that memcell is modified, the memregister and memaddress then update themselves to prepare for the next cycle. In short, we are replacing the tape, head, and control with memprocessors.

\subsection{UMM vs. LSM}

The mapping between a LSM and a UMM is fairly obvious. We simply have to map each ``reservoir cell'' to a memcell, and wire the circuit such that the point-wise separation and fading-memory properties are satisfied. The explicit construction of the circuit to realize such properties will not be explored here. 

Although, in theory, it is possible to simulate an LSM with a UMM, it is not always efficient or necessary to do so in practice. The circuit topologies of the two machines are very different, and they are designed to perform different tasks.

For the LSM model, the connections between the reservoir cells are typically random, and the reservoir as a whole is not trained. The expectation of getting the correct output relies entirely on training the output function correctly. In the end, the operation of the machine relies on statistical methods, and is inevitably prone to making errors. In some sense, the machine as a whole is analogous to a ``learning algorithm''~\cite{liquid_learn}.

On the other hand, for the UMM, we can connect the memcells into a circuit specific to the tasks at hand. One realization of this connection employs self-organizing logic gates~\cite{dmm2} to control the evolution of the machine such that it will always evolve towards an equilibrium state representing the solution to a particular problem (the machine is {\it deterministic}). 

In the general case, the UMM is an entirely different computing paradigm than the LSM, with the UMM being able to provide exponential speed-up for certain hard problems~\cite{dmm2}\cite{max-sat}. In other words, while possible, utilizing a UMM to simulate an LSM will not be exploiting all the properties of the UMM to its full use. In practical applications, it would then be more advantageous to use a UMM (and its digital version, a DMM) to tackle directly the same problems investigated by LSMs, without the intermediate step of simulating the LSM itself.

\subsection{UMM vs. Quantum Computers}

Simulating a quantum computer with a UMM requires ``compressing" the internal state of a quantum computer. Recall that a quantum computer has $m^n$ basis states, and each basis state is associated with some complex factor. The most obvious mapping is to map the $m^n$ basis states to $m^n$ mem-cells, and we can perform the simulation by wiring the mem-cells in a way such that the unitary evolution of interest is realizable. Of course, this would require an exponentially scaling number of mem-cells to achieve.

However, in this work, we have shown that any finite number of mem-cells can produce an internal state set with size $\beth_1$, which is also the size of any quantum computer. Therefore, in theory, we can map the basis states of the quantum computer to only $O(n)$ mem-cells, a linearly scaling number. However, in practice, this may come at the expense of exponentially deteriorating accuracy or exponentially scaling time.

This paper does not give a definitive answer as to whether it is possible for a UMM to simulate a quantum computer in polynomial time. Nevertheless, one of the features that makes UMMs a practical and powerful model of computation is precisely its ``information overhead''. Information overhead and quantum entanglement share some similarities: in some sense, both of them allow the machine to access a set of results of mathematical operations (without actually {\it storing} them) that is larger than that provided by simply the union of {\it non-interacting} processing units \cite{umm}\cite{shor}. We could then argue that we may exploit the information overhead property of a UMM to represent efficiently the entanglement of a quantum system. At this point, however, this question is still open. 

\section{Discussion}\label{discussions}

Within the set-theoretic framework of this paper, all the formulations are independent of the notion of resource. In other words, we showed that a UMM, in theory, can simulate any Turing machines, liquid-state machines, or quantum computers. However, performing such simulations with \textit{polynomial resources}, especially in the quantum case, cannot be proved within this framework. In fact, the question of whether one can perform the simulation of these machines using UMMs with polynomial resources is, of course, dependent on the specific algorithms that one is trying to simulate. 

Nevertheless, on a more practical note, the usefulness of UMMs, and in particular their digital realization (DMMs)~\cite{dmm2}, have already been shown to offer substantial advantages over standard algorithms, or even quantum computers, for a wide variety of hard combinatorial optimization problems. In fact, simulations of DMMs have already been applied to problems such as the maximum satisfiability (Max-SAT) \cite{max-sat,stress}, subset-sum \cite{leverage}, pre-training of neural networks \cite{haik}, and integer linear programming \cite{int-pro}, all of which carry great scientific and industrial value. In all these cases, DMM simulations have already outperformed traditional algorithmic approaches, in some cases providing an exponential speed up \cite{max-sat}. Therefore, if a physical DMM is realized through self-organizing logic gates and memristive elements \cite{dmm2}, the speed increase will be even more evident.

On the more theoretical end of the spectrum, this paper presents a simple but general set-theoretical methodology for checking equivalence or completeness relations between different computational models. This is useful in cases where other formal approaches are too cumbersome.

\section{Conclusions}\label{conclusions}

In conclusion, we have developed a set-theoretical approach to describe the relation between universal memcomputing machines and other types of computing models, in particular Turing machines, liquid-state machines, and quantum computers. Within this new mathematical framework, we have confirmed that UMMs are Turing-complete, a result already obtained in ~\cite{umm} using a different approach. 

In addition, we have also shown that UMMs are liquid-complete (or reservoir-complete) and quantum-complete, namely they can simulate {\it any} liquid-state (or reservoir-computing) machine and {\it any} quantum computer without reference to Turing machines. Of course, the results discussed here do not provide an answer to the question of what resources would be needed for a UMM to {\it efficiently} simulate such machines. Along these lines, it would be  interesting to study the relation between information overhead and quantum entanglement. If such a relation exists and can be exploited at a practical level, it may suggest how to utilize UMMs to efficiently simulate quantum problems that are currently believed to be only within reach of quantum computers (such as the efficient simulation of quantum Hamiltonians). Further work is however needed to address this practical question. 

\emph{Acknowledgments} -- MD acknowledges partial support from the Center for Memory and Recording Research at UCSD.

\bibliographystyle{ieeetr}
\bibliography{universality_of_umm}

\end{document}